\documentclass{article} % For LaTeX2e
\usepackage{iclr2027_conference,times}

\usepackage{amsmath,amsfonts,bm}

\def\eqref#1{equation~\ref{#1}}
\def\1{\bm{1}}

\DeclareMathAlphabet{\mathsfit}{\encodingdefault}{\sfdefault}{m}{sl}
\SetMathAlphabet{\mathsfit}{bold}{\encodingdefault}{\sfdefault}{bx}{n}

\usepackage{hyperref}
\usepackage{url}
\usepackage{microtype}
\usepackage{subcaption}
\usepackage{booktabs} % for professional tables
\usepackage{multirow}
\usepackage{colortbl}
\usepackage{xcolor}
\usepackage{multicol}
\usepackage{mdframed}
\usepackage{enumitem}
\usepackage{placeins}
\usepackage{float}
\usepackage[most]{tcolorbox}
\usepackage{tabularx}
\usepackage{booktabs}
\usepackage{array}
\usepackage{xcolor}
\usepackage{fvextra}
\usepackage{hyperref}

\usepackage{booktabs}
\usepackage{tabularx}
\usepackage{array}
\usepackage{makecell}
\usepackage[table]{xcolor}

\newcolumntype{Y}{>{\raggedright\arraybackslash}X}
\usepackage{algorithm}
\usepackage{algorithmic}
\usepackage{amsmath}
\usepackage{amssymb}
\usepackage{mathtools}
\usepackage{amsthm}
\usepackage{comment}
\newtheorem{theorem}{Theorem}

\usepackage{array}
\usepackage{booktabs}
\usepackage{tabularx}
\usepackage[table]{xcolor}

\usepackage{pifont}

\title{Is Human-Readable Text Necessary for Effective LLM Fine-Tuning?}

\author{
Jinhao Zhang$^{1}$ \quad Zeyu Liu$^{3}$ \quad Zicheng Yan$^{2}$ \quad
Yunquan Zhang$^{3}$ \quad Daning Cheng$^{3,\dagger}$ \quad Song Tang$^{4}$\\[0.6em]
{\normalfont\small $^{1}$Beijing University of Posts and Telecommunications}\\
{\normalfont\small $^{2}$University of Science and Technology of China}\\
{\normalfont\small $^{3}$Institute of Computing Technology, Chinese Academy of Sciences}\\
{\normalfont\small $^{4}$University of Shanghai for Science and Technology}\\[0.3em]
{\normalfont\small $^{\dagger}$Corresponding author}
}

\iclrfinalcopy % Reveal authors and omit review line numbers for the arXiv version.
\begin{document}

\maketitle
\lhead{Preprint}

\begin{abstract}
Is human readability necessary for effective fine-tuning of large language models? We investigate whether model-conditioned training representations can preserve or improve adaptation utility without requiring a human-readable textual form. We propose Desired-Update-Aligned Synthetic Data (DASA), which uses activation-gradient feedback from a frozen reference model to guide the optimization of continuous synthetic input embeddings. Inspired by the role of activation gradients in local risk reduction, DASA targets useful adaptation updates rather than source-text reconstruction or linguistic fluency. The resulting embeddings are used directly for downstream fine-tuning; discrete token projections are employed only for qualitative inspection. Experiments on six models from the Llama and Qwen families, ranging from 1B to 32B parameters, cover six benchmarks spanning knowledge, mathematical reasoning, code generation, and commonsense reasoning. Under matched LoRA adaptation settings, DASA achieves performance comparable to the source natural-language data and surpasses it in multiple configurations, while outperforming GRADMM in most comparisons. Further experiments cover general-domain and task-specialized source data. Under the evaluated synthesis settings, DASA provides a $3.6$--$4.9\times$ speedup over GRADMM with comparable peak GPU memory.  
\end{abstract}

\section{Introduction}
Human readability and training utility represent two distinct properties of fine-tuning data for large language models (LLMs). Readability concerns whether the data are coherent and interpretable to humans, whereas training utility concerns whether they provide effective learning signals that improve model performance on held-out tasks. Although natural-language data can exhibit both properties simultaneously, their co-occurrence does not imply that one is necessary for the other. In particular, it remains unclear whether human-readable expressions are also required for models to benefit from fine-tuning. This raises a more fundamental question about the objective of data design: rather than treating human-readable language as a default representation, can we construct training data directly around the target model's learning requirements? This work studies the following question: Is human readability necessary for effective fine-tuning, or can training utility be preserved when linguistic form is substantially relaxed?

We focus on post-training adaptation of existing language models rather than pretraining from scratch. Concretely, given a reference model ($f_{\theta_0}$) and a natural-language source dataset (D), we aim to construct an alternative dataset (D') that no longer relies on the conventional human-readable form of the original text, while still providing comparable adaptation gains under the same fine-tuning configuration and training budget. We further ask whether alternative representations constructed specifically for the receiving model can, in some settings, outperform the original natural-language data. Crucially, we define training utility by downstream performance after fine-tuning on independent evaluation tasks. We do not require the model to interpret, reconstruct, or translate the alternative sequences back into natural language. In this sense, the object of interest is not whether the representation remains semantically transparent to humans, but whether it serves as an effective substrate for model adaptation.

Prior work has shown that, under appropriate conditions, alignment between activation gradients and newly introduced residual directions can provide a basis for constructing local updates that reduce risk. Inspired by this perspective, we bring the idea of exploiting model-internal improvement signals into the construction of training data. Rather than requiring alternative data to preserve the human-readable form of the original text, we focus on whether it retains the effective signals needed for model adaptation. To this end, we use internal feedback produced by a reference model on the source data to guide the optimization of alternative training sequences. Our goal is to investigate whether natural-language data can be transformed into representations that are difficult for humans to read directly, while still retaining high utility for fine-tuning. This design allows us to test whether the adaptation value of data for an existing model can be preserved, or even improved, while substantially changing its surface form.

In this work, we propose  Desired-Update-Aligned Synthetic Data (DASA), a method for constructing alternative training data guided by the alignment of model-internal signals. Given source data, DASA treats the internal signals induced in a reference model as optimization targets and updates alternative sequences through gradient-based optimization so that the signals they elicit align with the desired construction objective. Importantly, the procedure does not explicitly optimize for reconstructing the source text or producing fluent natural language. This allows the resulting sequences to depart substantially from the original word order, syntax, and surface form, yielding training representations that are difficult for humans to read directly. The central idea is not to treat unreadability itself as a mechanism for improving performance, but to relax constraints on linguistic form and shift the optimization objective toward internal signals useful to the model. At the same time, we do not assume that similarity in internal signals implies equivalence of the resulting training dynamics. Instead, we evaluate the validity of this construction criterion through the actual downstream performance obtained after fine-tuning on the alternative data.

We evaluate DASA on the Llama and Qwen model families, covering six models from 1B to 32B parameters across a diverse range of downstream tasks. Under a unified fine-tuning setup, DASA-generated alternative data yields adaptation gains comparable to those obtained with the original natural-language data and surpasses it on multiple model–task combinations. Compared with the gradient-matching baseline GRADMM, DASA performs better in the majority of evaluated settings while requiring less time for synthetic data generation. Additional experiments with both general-domain corpora and task-relevant data show that these advantages persist across different data sources, supporting the effectiveness and practicality of constructing training representations from model-internal signals.

Our main contributions are: 
\ding{182} We propose DASA, which constructs fine-tuning data by aligning model-internal signals without requiring text reconstruction or language fluency. 
\ding{183} We show that DASA preserves fine-tuning utility across model scales, tasks, and data budgets, and can outperform the original text in some settings. 
\ding{184} We evaluate DASA across different data sources, repeated runs, and generation costs to assess its effectiveness and practicality.

\section{Related Work}
\label{sec:related_work}
Dataset distillation aims to preserve training utility with a small set of synthetic examples, and methods developed for images have since been extended to text~\citep{wang2018dataset,zhao2021gradient,zhao2021dsa,zhao2023distribution,cazenavette2022trajectory,DBLP:conf/icml/NguyenLBMRM25}. Text methods include continuous embedding optimization, discrete text generation, and training trajectory matching~\citep{sucholutsky2021soft,maekawa2024dilm,yao2025transferable}. GRADMM matches parameter gradients and, in its main procedure, projects optimized embeddings into low-perplexity text; it also reports an ablation that replaces top-$k$ filtering with $L_2$ vocabulary projection~\citep{DBLP:conf/icml/NguyenLBMRM25}. 

Prior work has also explored fine-tuning with difficult-to-read instructions while retaining the original natural-language responses~\citep{duan2025unnatural}, learnable prompts~\citep{shin2020autoprompt,li2021prefix,lester2021prompt,liu2022ptuning}, model-internal information for data synthesis~\citep{yin2020deepinversion,wang2022cafe}, and theoretical analyses of local update directions~\citep{cheng2026scaling}. Adapters and LoRA provide mechanisms for efficient model adaptation~\citep{houlsby2019adapters,hu2022lora}. DASA constructs continuous synthetic inputs for subsequent fine-tuning by optimizing the first-order effect of their induced update direction on reference risk, without requiring source-text reconstruction or linguistic fluency. The residual expansion analysis motivates this objective but does not guarantee the outcome of finite-step fine-tuning.

\section{methodology}
\subsection{Start Point}

The notation, basic setup, and standing assumptions used in this section
are given in Appendix~\ref{app:setup_notation}; additional conditions are
stated where needed. Appendix~\ref{app:prior_work} provides the relevant
definitions and derivations from this prior work.

At a fixed layer of the reference model, the activation gradient
$\zeta$ measures the first-order sensitivity of the loss to a change
in the hidden representation $z$. In their analysis of residual-block
expansion, \citet{cheng2026scaling} consider a realizable parameter
direction $v$ whose first-order effect on population risk is
\begin{equation}
D_Q(v)
=
\left.
\frac{\mathrm d}{\mathrm d\eta}
\mathbb E_Q\!\left[
\ell\bigl(f_{\mathrm{top}}(z+h_{\eta v}(z)),y\bigr)
\right]
\right|_{\eta=0}
=
\mathbb E_Q\!\left[\zeta^\top u_v(z)\right].
\end{equation}
Their Condition~1 requires the existence of such a direction with
$D_Q(v)<0$. Under the stated differentiability and realizability
conditions, their Theorem~4 then establishes that, for some sufficiently
small $\eta>0$,
\begin{equation}
\mathbb E_Q\!\left[
\ell\bigl(f_{\mathrm{top}}(z+h_{\eta v}(z)),y\bigr)
\right]
<
\mathbb E_Q\!\left[
\ell\bigl(f_{\mathrm{top}}(z),y\bigr)
\right].
\end{equation}
This is an existence result for the expanded model class; it does not
assert that a particular dataset or optimizer will find the improving
direction. Their Theorem~5 additionally bounds the probability of
nonpositive alignment between independent training and test
activation-gradient means under a nonzero population mean and
covariance control.

These results motivate a reverse data-construction question: given a
target model and reference data, can we construct synthetic inputs whose
induced update direction is locally improving on the reference set?
DASA turns this directional criterion into a data-construction objective
by optimizing continuous synthetic inputs $Z^s$ such that the
parameter-space update $v(Z^s)$ they induce produces a favorable
first-order change in the reference loss.

For the reference set
$\mathcal D_{\mathrm{ref}}
=\{(x_i,y_i)\}_{i=1}^{N_{\mathrm{ref}}}$, we compute and cache
\[
z_i=f_{\mathrm{bot}}(x_i),
\qquad
\zeta_i=\nabla_{z_i}\ell(f_{\mathrm{top}}(z_i),y_i).
\]
Replacing the population expectation in $D_Q(v)$ with the empirical
average gives the first-order change in reference loss along $v$:
\begin{equation}
D_{\mathcal D_{\mathrm{ref}}}(v)
=
\frac{1}{N_{\mathrm{ref}}}
\sum_{i=1}^{N_{\mathrm{ref}}}
\zeta_i^\top u_v(z_i).
\end{equation}

At each synthesis step, the current synthetic examples and their retained
prediction targets induce a parameter-space direction $v(Z^s)$.
We then evaluate the first-order effect of this direction on the fixed
reference examples and optimize
\begin{equation}
\boxed{
\min_{Z^s}
D_{\mathcal D_{\mathrm{ref}}}\bigl(v(Z^s)\bigr)
=
\min_{Z^s}
\frac{1}{N_{\mathrm{ref}}}
\sum_{i=1}^{N_{\mathrm{ref}}}
\zeta_i^\top u_{v(Z^s)}(z_i).
}
\end{equation}
Gradients of this objective are propagated through $v(Z^s)$ to update
$Z^s$. Thus, DASA directly optimizes the directional derivative of the
reference loss induced by the candidate synthetic inputs, rather than a
distance between activation-gradient means. The relation between a LoRA
parameter direction and $u_v(z)$ is detailed in
Appendix~\ref{app:covariance_control}. Under the stated conditions, a
negative objective value indicates that the induced update direction is
locally improving on the reference set; downstream performance after
finite-step adaptation is evaluated empirically.

\subsection{Algorithm Design}
\begin{algorithm}[t]
\caption{DASA: Desired-Update-Aligned Synthetic Data}
\label{alg:dasa}
\begin{algorithmic}[1]
\REQUIRE Frozen model $f_{\theta_0}$,
reference dataset
$\mathcal D_{\mathrm{ref}}
=\{(x_i,y_i)\}_{i=1}^{N_{\mathrm{ref}}}$,
target layer $l^*$,
synthetic budget $n\times L$,
steps $T_{\mathrm{syn}}$,
learning rate $\eta_{\mathrm{syn}}$
\ENSURE Continuous synthetic inputs $S_{\mathrm{syn}}$
and shifted-token targets $Y^s$

\FOR{$i=1,\ldots,N_{\mathrm{ref}}$}
    \STATE Compute
    $z_i\leftarrow f_{\mathrm{bot}}(x_i)$
    and
    $\zeta_i\leftarrow
    \nabla_{z_i}\ell(f_{\mathrm{top}}(z_i),y_i)$.
\ENDFOR

\STATE Define
$D_{\mathcal D_{\mathrm{ref}}}(v)
\coloneqq
\frac{1}{N_{\mathrm{ref}}}
\sum_{i=1}^{N_{\mathrm{ref}}}
\zeta_i^\top u_v(z_i)$.

\STATE Select $n$ length-$L$ reference token sequences $X^{(0)}$.
\STATE Initialize $Z^s$ from the input embeddings of $X^{(0)}$
and set $Y^s\leftarrow\operatorname{Shift}(X^{(0)})$.

\FOR{$t=1,\ldots,T_{\mathrm{syn}}$}
    \STATE Compute the tangent direction
    $v\leftarrow v(Z^s)$
    clip its $\ell_2$ norm to at most $2.0$.
    \STATE Evaluate $D_{\mathcal D_{\mathrm{ref}}}(v)$.
    \STATE Update
    $Z^s\leftarrow Z^s-\eta_{\mathrm{syn}}
    \nabla_{Z^s}
    D_{\mathcal D_{\mathrm{ref}}}\bigl(v(Z^s)\bigr)$.
\ENDFOR

\STATE $S_{\mathrm{syn}}\leftarrow\mathrm{BF16}(Z^s)$.
\RETURN $(S_{\mathrm{syn}},Y^s)$
\end{algorithmic}
\end{algorithm}

DASA takes a frozen pretrained model $f_{\theta_0}$, a finite reference set
$\mathcal D_{\mathrm{ref}}=\{(x_i,y_i)\}_{i=1}^{N_{\mathrm{ref}}}$,
a target layer $l^*$, and a synthetic budget of $n\times L$.
Its output is a continuous input tensor for subsequent fine-tuning, rather
than a set of natural-language sentences. We denote the tensor being optimized
by $Z^s\in\mathbb R^{n\times L\times d_{\mathrm{emb}}}$, representing $n$
synthetic sequences of $L$ continuous input embeddings each.
Let $X^{(0)}=\{x_j^{(0)}\}_{j=1}^{n}$ denote the length-$L$
reference token sequences used for initialization. We initialize
$Z_0^s$ from their input embeddings and set
$Y^s=\operatorname{Shift}(X^{(0)})$, where
$y^s_{j,t}=x^{(0)}_{j,t+1}$ for $t=1,\ldots,L-1$.
These are the standard shifted-token targets of causal language
modeling, not separately annotated labels. For clarity, Algorithm 1 presents the single-layer case; in experiments, we apply the same procedure to all Transformer layers, clipping the induced direction separately at each layer.
\paragraph{Synthetic data generation.}
Algorithm~\ref{alg:dasa} summarizes the synthesis procedure. We first pass
the reference data through the frozen model and compute their activation
gradients at the target layer. The synthetic activation gradients determine
the induced direction $v(Z^s)$. To prevent the objective from being reduced
merely by increasing the magnitude of the induced update, we clip its
$\ell_2$ norm to a maximum value of $2.0$.
We then evaluate the first-order effect of $\tilde v(Z^s)$ on the reference
risk using the objective defined in Section~3.1. This objective is
backpropagated to $Z^s$, which is updated with SGD. The model parameters and
prediction targets remain fixed throughout synthesis. After optimization, we
convert $Z^s$ to BF16 to obtain the continuous synthetic data
$S_{\mathrm{syn}}$.

\paragraph{Downstream adaptation and evaluation.}
We pair $S_{\mathrm{syn}}$ with the retained prediction targets and use the
continuous embeddings directly for fine-tuning, without projecting them to
discrete tokens. At this stage, the synthetic data remain fixed while LoRA
updates the model parameters. We evaluate the adapted model on independent
tasks and compare it with fine-tuning on the original text under the same
adaptation configuration and training budget.

\section{Experiments}
We evaluate DASA across multiple model scales, downstream tasks, calibration budgets, and data sources. Beyond overall adaptation performance, we examine robustness across random seeds, computational efficiency, sensitivity to optimization settings, and controlled baselines that isolate the effect of DASA's update-alignment objective. These experiments are designed to assess both the effectiveness and robustness of model-internal-signal-driven synthetic data construction.
\subsection{Experimental Setup}
\textbf{Models and Benchmarks.}
We evaluate our method on six instruction-tuned language models spanning two model families and multiple scales: Llama-3.2-1B-Instruct, Llama-3.2-3B-Instruct \citep{meta2024llama32}, Llama-3.1-8B-Instruct \citep{DBLP:journals/corr/abs-2407-21783}, and Qwen3-8B/14B/32B \citep{DBLP:journals/corr/abs-2505-09388}. We consider six downstream benchmarks, including MMLU \citep{DBLP:conf/iclr/HendrycksBBZMSS21}, GSM8K \citep{DBLP:journals/corr/abs-2110-14168}, MATH \citep{DBLP:conf/nips/HendrycksBKABTS21}, MBPP \citep{DBLP:journals/corr/abs-2108-07732}, ARC-Challenge (ARC-C) \citep{DBLP:journals/corr/abs-1803-05457}, and CommonsenseQA (CSQA) \citep{DBLP:conf/naacl/TalmorHLB19}.

\textbf{Baselines.}
We compare against three baselines. \textbf{Base} denotes the original model without downstream adaptation. \textbf{Real-text} fine-tunes the model using natural-language examples under the same data budget as our synthetic data. \textbf{GRADMM} \citep{DBLP:conf/icml/NguyenLBMRM25} constructs synthetic training examples through parameter-gradient matching. In contrast, \textbf{Ours} directly optimizes synthetic sequences such that the induced activation-gradient update aligns with the desired update direction. Throughout all experiments, including the computational-efficiency comparison, both DASA and our implementation of GRADMM use all Transformer layers and pass the optimized continuous embeddings directly to downstream LoRA adaptation without vocabulary projection. Real-text uses the same sampled reference sequences \(X^{(0)}\) that initialize DASA.

\textbf{LoRA Adaptation and Data Budget.}
We freeze the pretrained backbone and optimize LoRA parameters only. All methods use identical LoRA \citep{DBLP:conf/iclr/HuSWALWWC22}, optimization, and evaluation settings. We consider three data budgets, $50\times2048$, $100\times2048$, and $200\times2048$, where the first term is the number of sequences and 2048 is the maximum sequence length. 

\textbf{Evaluation protocol.}
We use accuracy for MMLU, ARC-Challenge, and CommonsenseQA; exact-match
accuracy for GSM8K and MATH after standard answer extraction, and pass@1
for MBPP using execution-based functional correctness. All methods are
evaluated with identical prompts, few-shot demonstrations, decoding
parameters, and evaluation scripts. Generative tasks use deterministic
decoding, so differences across methods arise from adaptation rather
than sampling variation.
\subsection{Main Results}
Table~\ref{tab:main_results} reports results across six models, six tasks, and three calibration budgets. DASA consistently improves over both the unadapted model and the parameter-gradient-matching baseline across model families and scales. In particular, compared with GRADMM, DASA achieves higher performance in 31 of 36 model--task configurations under the $50\times2048$ budget, 35 of 36 under $100\times2048$, and 35 of 36 under $200\times2048$. The advantage therefore becomes more consistent as the available synthetic-data budget increases rather than disappearing once additional training examples are introduced.

\textbf{Effectiveness across model scales.}
The improvements are observed from 1B-scale models to Qwen3-32B. Under the $200\times2048$ budget, for example, DASA improves over GRADMM on six of six tasks for Llama-3.2-1B and Llama-3.1-8B, while outperforming GRADMM on all six benchmarks for both Qwen3-14B and Qwen3-32B. On Qwen3-32B, DASA improves MMLU from 84.65 to 85.26, GSM8K from 96.36 to 97.19, MATH from 87.56 to 88.28, ARC-C from 71.76 to 72.44, and CSQA from 85.83 to 86.65. The same qualitative behavior on substantially different model scales suggests that activation-level update realization is not tied to a particular parameter regime.

\textbf{Effectiveness across capabilities.}
DASA also provides gains across qualitatively different task families. On Llama-3.2-3B with the largest budget, DASA improves MATH from 54.12 to 55.16 and MBPP from 71.2 to 72.0 over GRADMM, while improving MMLU from 61.20 to 61.96 and CSQA from 76.41 to 77.23. Similarly, on Qwen3-8B, it improves MATH from 86.44 to 87.40, MBPP from 62.2 to 63.0, and ARC-C from 66.30 to 67.24. These results indicate that the gains are not concentrated in one evaluation category, but extend to knowledge, reasoning, coding, and commonsense tasks.

\textbf{Effect of the calibration budget.}
Increasing the calibration budget generally strengthens downstream adaptation for all synthesis methods, but DASA benefits particularly consistently from additional calibration data. Its relative number of wins over GRADMM increases from 31/36 to 35/36 as the budget grows from $50\times2048$ to $200\times2048$. This observation is important because a synthetic-data method may perform well at extremely small budgets simply by exploiting a few favorable examples. Instead, DASA continues to produce useful updates as the synthetic set grows, suggesting that activation alignment provides a stable optimization target beyond the low-data regime.

An additional feature of this experiment is that all downstream tasks use calibration samples drawn from the same cleaned C4 corpus. The model is therefore not simply trained on natural examples drawn from the target benchmark distribution. Nevertheless, the induced updates consistently improve task performance, supporting the central premise that the update signal itself, rather than direct semantic correspondence between calibration text and evaluation examples, can drive useful downstream adaptation.

\begin{table}[H]
\caption{\textbf{Main results across models and tasks.}
Base, Real-text, GRADMM, and DASA under three C4 calibration budgets; all adaptation methods use identical LoRA and optimization settings.}
\label{tab:main_results}
\centering
\tiny
\renewcommand{\arraystretch}{0.85}
\setlength{\tabcolsep}{1.5pt}

\resizebox{\textwidth}{!}{%
\begin{tabular}{l *{18}{c}}
\toprule
& \multicolumn{18}{c}{
\cellcolor{gray!15}
\textbf{Token Budget: 50$\times$2048}
} \\
\cmidrule(lr){2-19}
& \multicolumn{6}{c}{\cellcolor{gray!10}\textbf{Llama-3.2-1B}}
& \multicolumn{6}{c}{\cellcolor{gray!10}\textbf{Llama-3.2-3B}}
& \multicolumn{6}{c}{\cellcolor{gray!10}\textbf{Llama-3.1-8B}} \\
\cmidrule(lr){2-7}
\cmidrule(lr){8-13}
\cmidrule(lr){14-19}
\textbf{Method}
& MMLU & G8K & MATH & MBPP & ARC & CSQA
& MMLU & G8K & MATH & MBPP & ARC & CSQA
& MMLU & G8K & MATH & MBPP & ARC & CSQA \\
\midrule
Base
& 45.79 & 46.32 & 32.70 & 49.6 & 54.61 & 59.13
& 58.62 & 82.03 & 51.40 & 68.4 & 73.72 & 73.87
& 65.41 & 85.90 & 50.70 & 75.4 & 82.08 & 75.92 \\

Real-text
& 45.91
& 46.55
& 32.96
& 50.0
& \cellcolor{blue!15}\textbf{54.86}
& 59.30
& 58.76
& 82.18
& 51.60
& 68.8
& 73.81
& \cellcolor{blue!15}\textbf{74.20}
& 65.45
& \cellcolor{blue!15}\textbf{86.20}
& 50.88
& 75.6
& 82.34
& 76.00 \\

GRADMM
& 46.10
& 47.08
& 33.40
& 50.4
& 54.35
& 59.54
& 58.98
& 82.56
& 51.20
& 69.0
& 74.06
& 74.20
& 65.72
& 86.15
& 51.16
& 76.0
& 82.42
& 76.17 \\

\textbf{Ours}
& \cellcolor{blue!15}\textbf{46.46}
& \cellcolor{blue!15}\textbf{46.63}
& \cellcolor{blue!15}\textbf{34.20}
& \cellcolor{blue!15}\textbf{50.8}
& 54.69
& \cellcolor{blue!15}\textbf{59.95}
& \cellcolor{blue!15}\textbf{59.22}
& \cellcolor{blue!15}\textbf{82.87}
& \cellcolor{blue!15}\textbf{51.66}
& \cellcolor{blue!15}\textbf{69.6}
& \cellcolor{blue!15}\textbf{74.32}
& 74.04
& \cellcolor{blue!15}\textbf{65.56}
& 86.05
& \cellcolor{blue!15}\textbf{51.56}
& \cellcolor{blue!15}\textbf{76.4}
& \cellcolor{blue!15}\textbf{82.68}
& \cellcolor{blue!15}\textbf{76.41} \\

\midrule
& \multicolumn{6}{c}{\cellcolor{gray!10}\textbf{Qwen3-8B}}
& \multicolumn{6}{c}{\cellcolor{gray!10}\textbf{Qwen3-14B}}
& \multicolumn{6}{c}{\cellcolor{gray!10}\textbf{Qwen3-32B}} \\
\cmidrule(lr){2-7}
\cmidrule(lr){8-13}
\cmidrule(lr){14-19}

\textbf{Method}
& MMLU & G8K & MATH & MBPP & ARC & CSQA
& MMLU & G8K & MATH & MBPP & ARC & CSQA
& MMLU & G8K & MATH & MBPP & ARC & CSQA \\
\midrule

Base
& 74.90 & 91.89 & 83.98 & 59.6 & 64.33 & 81.98
& 78.85 & 93.86 & 84.76 & 64.6 & 66.98 & 81.00
& 82.11 & 94.54 & 85.26 & 65.8 & 69.28 & 84.03 \\

Real-text
& 75.05
& 92.04
& 84.12
& 59.8
& \cellcolor{blue!15}\textbf{64.68}
& 81.98
& 78.93
& \cellcolor{blue!15}\textbf{94.16}
& 84.88
& 64.8
& 67.15
& \cellcolor{blue!15}\textbf{81.33}
& \cellcolor{blue!15}\textbf{82.30}
& 94.69
& 85.34
& 66.2
& 69.37
& 84.28 \\

GRADMM
& 75.24
& 92.27
& 84.32
& 60.2
& 64.68
& 82.23
& 79.04
& 93.31
& 85.10
& 65.0
& 67.32
& 81.07
& 82.11
& 94.69
& 85.48
& 66.2
& 69.54
& 84.28 \\

\textbf{Ours}
& \cellcolor{blue!15}\textbf{75.49}
& \cellcolor{blue!15}\textbf{92.49}
& \cellcolor{blue!15}\textbf{84.60}
& \cellcolor{blue!15}\textbf{60.8}
& 64.59
& \cellcolor{blue!15}\textbf{82.31}
& \cellcolor{blue!15}\textbf{79.08}
& 94.09
& \cellcolor{blue!15}\textbf{85.30}
& \cellcolor{blue!15}\textbf{65.6}
& \cellcolor{blue!15}\textbf{67.49}
& 81.24
& 82.24
& \cellcolor{blue!15}\textbf{94.77}
& \cellcolor{blue!15}\textbf{85.64}
& \cellcolor{blue!15}\textbf{66.6}
& \cellcolor{blue!15}\textbf{69.71}
& \cellcolor{blue!15}\textbf{84.52} \\

\midrule
\midrule
& \multicolumn{18}{c}{
\cellcolor{gray!15}
\textbf{Token Budget: 100$\times$2048}
} \\
\cmidrule(lr){2-19}
& \multicolumn{6}{c}{\cellcolor{gray!10}\textbf{Llama-3.2-1B}}
& \multicolumn{6}{c}{\cellcolor{gray!10}\textbf{Llama-3.2-3B}}
& \multicolumn{6}{c}{\cellcolor{gray!10}\textbf{Llama-3.1-8B}} \\
\cmidrule(lr){2-7}
\cmidrule(lr){8-13}
\cmidrule(lr){14-19}

\textbf{Method}
& MMLU & G8K & MATH & MBPP & ARC & CSQA
& MMLU & G8K & MATH & MBPP & ARC & CSQA
& MMLU & G8K & MATH & MBPP & ARC & CSQA \\
\midrule

Base
& 45.79 & 46.32 & 32.70 & 49.6 & 54.61 & 59.13
& 58.62 & 82.03 & 51.40 & 68.4 & 73.72 & 73.87
& 65.41 & 85.90 & 50.70 & 75.4 & 82.08 & 75.92 \\

Real-text
& 46.05
& \cellcolor{blue!15}\textbf{46.93}
& 33.12
& 50.0
& \cellcolor{blue!15}\textbf{54.86}
& 59.38
& 58.83
& 82.34
& \cellcolor{blue!15}\textbf{51.72}
& 68.8
& 73.89
& 74.12
& \cellcolor{blue!15}\textbf{65.59}
& 86.20
& 51.00
& 75.6
& 82.25
& 76.09 \\

GRADMM
& 46.44 & 46.69 & 33.88 & 50.8 & 54.29 & 59.79
& 59.20 & 82.94 & 51.32 & 69.4 & 74.23 & 74.45
& 65.11 & 86.66 & 51.48 & 76.2 & 82.51 & 76.41 \\

\textbf{Ours}
& \cellcolor{blue!15}\textbf{47.21}
& 46.78
& \cellcolor{blue!15}\textbf{35.02}
& \cellcolor{blue!15}\textbf{52.0}
& 54.69
& \cellcolor{blue!15}\textbf{60.69}
& \cellcolor{blue!15}\textbf{59.91}
& \cellcolor{blue!15}\textbf{84.00}
& 51.60
& \cellcolor{blue!15}\textbf{70.8}
& \cellcolor{blue!15}\textbf{75.00}
& \cellcolor{blue!15}\textbf{75.18}
& 65.50
& \cellcolor{blue!15}\textbf{87.49}
& \cellcolor{blue!15}\textbf{52.38}
& \cellcolor{blue!15}\textbf{77.4}
& \cellcolor{blue!15}\textbf{83.19}
& \cellcolor{blue!15}\textbf{77.07} \\

\midrule
& \multicolumn{6}{c}{\cellcolor{gray!10}\textbf{Qwen3-8B}}
& \multicolumn{6}{c}{\cellcolor{gray!10}\textbf{Qwen3-14B}}
& \multicolumn{6}{c}{\cellcolor{gray!10}\textbf{Qwen3-32B}} \\
\cmidrule(lr){2-7}
\cmidrule(lr){8-13}
\cmidrule(lr){14-19}

\textbf{Method}
& MMLU & G8K & MATH & MBPP & ARC & CSQA
& MMLU & G8K & MATH & MBPP & ARC & CSQA
& MMLU & G8K & MATH & MBPP & ARC & CSQA \\
\midrule

Base
& 74.90 & 91.89 & 83.98 & 59.6 & 64.33 & 81.98
& 78.85 & 93.86 & 84.76 & 64.6 & 66.98 & 81.00
& 82.11 & 94.54 & 85.26 & 65.8 & 69.28 & 84.03 \\

Real-text
& 75.07
& 92.12
& 84.18
& 59.8
& 64.51
& \cellcolor{blue!15}\textbf{82.15}
& \cellcolor{blue!15}\textbf{79.00}
& 94.01
& 84.92
& 64.8
& 67.15
& 81.16
& 82.23
& \cellcolor{blue!15}\textbf{94.69}
& 85.40
& 66.0
& 69.37
& 84.11 \\

GRADMM
& 75.37 & 92.42 & 84.56 & 60.4 & 64.76 & 82.09
& 78.27 & 94.31 & 85.24 & 65.4 & 67.41 & 81.41
& 82.45 & 94.02 & 85.64 & 66.4 & 69.62 & 84.36 \\

\textbf{Ours}
& \cellcolor{blue!15}\textbf{75.98}
& \cellcolor{blue!15}\textbf{93.10}
& \cellcolor{blue!15}\textbf{85.24}
& \cellcolor{blue!15}\textbf{61.8}
& \cellcolor{blue!15}\textbf{65.44}
& 82.06
& 78.94
& \cellcolor{blue!15}\textbf{94.84}
& \cellcolor{blue!15}\textbf{85.84}
& \cellcolor{blue!15}\textbf{66.4}
& \cellcolor{blue!15}\textbf{68.00}
& \cellcolor{blue!15}\textbf{81.98}
& \cellcolor{blue!15}\textbf{82.91}
& 94.62
& \cellcolor{blue!15}\textbf{86.16}
& \cellcolor{blue!15}\textbf{67.4}
& \cellcolor{blue!15}\textbf{70.14}
& \cellcolor{blue!15}\textbf{84.85} \\

\midrule
\midrule
& \multicolumn{18}{c}{
\cellcolor{gray!25}
\textbf{Token Budget: 200$\times$2048}
} \\
\cmidrule(lr){2-19}
& \multicolumn{6}{c}{\cellcolor{gray!10}\textbf{Llama-3.2-1B}}
& \multicolumn{6}{c}{\cellcolor{gray!10}\textbf{Llama-3.2-3B}}
& \multicolumn{6}{c}{\cellcolor{gray!10}\textbf{Llama-3.1-8B}} \\
\cmidrule(lr){2-7}
\cmidrule(lr){8-13}
\cmidrule(lr){14-19}

\textbf{Method}
& MMLU & G8K & MATH & MBPP & ARC & CSQA
& MMLU & G8K & MATH & MBPP & ARC & CSQA
& MMLU & G8K & MATH & MBPP & ARC & CSQA \\
\midrule

Base
& 45.79 & 46.32 & 32.70 & 49.6 & 54.61 & 59.13
& 58.62 & 82.03 & 51.40 & 68.4 & 73.72 & 73.87
& 65.41 & 85.90 & 50.70 & 75.4 & 82.08 & 75.92 \\

Real-text
& \cellcolor{blue!15}\textbf{48.01}
& 48.67
& 34.88
& 51.8
& 56.91
& 61.26
& 60.73
& \cellcolor{blue!15}\textbf{84.23}
& 53.76
& 70.6
& \cellcolor{blue!15}\textbf{75.77}
& 76.17
& 67.49
& 88.10
& \cellcolor{blue!15}\textbf{52.92}
& 77.6
& 84.13
& 78.13 \\

GRADMM
& 47.43 & 48.45 & 35.20 & 52.4 & 56.57 & 61.59
& 61.20 & 84.08 & 54.12 & 71.2 & 75.51 & 76.41
& 67.14 & 88.48 & 52.24 & 78.0 & 83.96 & 78.38 \\

\textbf{Ours}
& 47.88
& \cellcolor{blue!15}\textbf{48.90}
& \cellcolor{blue!15}\textbf{36.12}
& \cellcolor{blue!15}\textbf{53.4}
& \cellcolor{blue!15}\textbf{57.17}
& \cellcolor{blue!15}\textbf{62.49}
& \cellcolor{blue!15}\textbf{61.96}
& 84.08
& \cellcolor{blue!15}\textbf{55.16}
& \cellcolor{blue!15}\textbf{72.0}
& 75.60
& \cellcolor{blue!15}\textbf{77.23}
& \cellcolor{blue!15}\textbf{67.63}
& \cellcolor{blue!15}\textbf{89.46}
& 52.82
& \cellcolor{blue!15}\textbf{79.0}
& \cellcolor{blue!15}\textbf{84.73}
& \cellcolor{blue!15}\textbf{79.28} \\

\midrule
& \multicolumn{6}{c}{\cellcolor{gray!10}\textbf{Qwen3-8B}}
& \multicolumn{6}{c}{\cellcolor{gray!10}\textbf{Qwen3-14B}}
& \multicolumn{6}{c}{\cellcolor{gray!10}\textbf{Qwen3-32B}} \\
\cmidrule(lr){2-7}
\cmidrule(lr){8-13}
\cmidrule(lr){14-19}

\textbf{Method}
& MMLU & G8K & MATH & MBPP & ARC & CSQA
& MMLU & G8K & MATH & MBPP & ARC & CSQA
& MMLU & G8K & MATH & MBPP & ARC & CSQA \\
\midrule

Base
& 74.90 & 91.89 & 83.98 & 59.6 & 64.33 & 81.98
& 78.85 & 93.86 & 84.76 & 64.6 & 66.98 & 81.00
& 82.11 & 94.54 & 85.26 & 65.8 & 69.28 & 84.03 \\

Real-text
& 76.97
& 93.93
& 86.10
& 61.8
& 66.47
& \cellcolor{blue!15}\textbf{84.11}
& \cellcolor{blue!15}\textbf{80.89}
& 95.91
& 86.84
& 66.8
& 69.11
& 83.05
& 84.12
& 96.59
& 87.32
& \cellcolor{blue!15}\textbf{68.0}
& 71.33
& 86.08 \\

GRADMM
& 77.41 & 93.63 & 86.44 & 62.2 & 66.30 & 83.36
& 80.51 & 96.29 & 87.32 & 66.4 & 69.37 & 82.72
& 84.65 & 96.36 & 87.56 & 67.6 & 71.76 & 85.83 \\

\textbf{Ours}
& \cellcolor{blue!15}\textbf{78.09}
& \cellcolor{blue!15}\textbf{94.16}
& \cellcolor{blue!15}\textbf{87.40}
& \cellcolor{blue!15}\textbf{63.0}
& \cellcolor{blue!15}\textbf{67.24}
& 83.95
& 80.81
& \cellcolor{blue!15}\textbf{97.12}
& \cellcolor{blue!15}\textbf{88.08}
& \cellcolor{blue!15}\textbf{67.2}
& \cellcolor{blue!15}\textbf{70.14}
& \cellcolor{blue!15}\textbf{83.62}
& \cellcolor{blue!15}\textbf{85.26}
& \cellcolor{blue!15}\textbf{97.19}
& \cellcolor{blue!15}\textbf{88.28}
& 67.8
& \cellcolor{blue!15}\textbf{72.44}
& \cellcolor{blue!15}\textbf{86.65} \\

\bottomrule
\end{tabular}%
}
\end{table}
\subsection{Robustness Across Random Seeds}

The main experiment covers many model--task combinations and therefore reports single-run results. To assess run-to-run variation, we evaluate representative configurations over three independent runs on MMLU, GSM8K, and MBPP for Qwen3-8B and Qwen3-14B under the $200\times2048$ data budget. Table~\ref{tab:multiseed} reports the mean and standard deviation.

\begin{table}[H]
\caption{\textbf{Robustness across random seeds.} Mean $\pm$ standard deviation over three independent runs under the $200\times2048$ data budget. Random uses Gaussian inputs matched to the mean and standard deviation of DASA's initialization embeddings.}
\label{tab:multiseed}
\centering
\scriptsize
\renewcommand{\arraystretch}{0.98}
\setlength{\tabcolsep}{3.5pt}
\begin{tabular}{l *{6}{c}}
\toprule
& \multicolumn{3}{c}{\cellcolor{gray!10}\textbf{Qwen3-8B}}
& \multicolumn{3}{c}{\cellcolor{gray!10}\textbf{Qwen3-14B}} \\
\cmidrule(lr){2-4}
\cmidrule(lr){5-7}
\textbf{Method}
& \textbf{MMLU} $\uparrow$
& \textbf{GSM8K} $\uparrow$
& \textbf{MBPP} $\uparrow$
& \textbf{MMLU} $\uparrow$
& \textbf{GSM8K} $\uparrow$
& \textbf{MBPP} $\uparrow$ \\
\midrule
Random
& $71.8 \pm 0.35$
& $87.9 \pm 0.48$
& $55.2 \pm 0.8$
& $77.9 \pm 0.31$
& $90.1 \pm 0.41$
& $62.6 \pm 0.6$ \\
Real-text
& $76.97 \pm 0.13$
& $\mathbf{93.93 \pm 0.21}$
& $61.8 \pm 0.6$
& $80.89 \pm 0.17$
& $95.91 \pm 0.22$
& $\mathbf{66.8 \pm 0.4}$ \\
GRADMM
& $77.41 \pm 0.15$
& $93.63 \pm 0.24$
& $62.2 \pm 0.6$
& $80.51 \pm 0.16$
& $96.29 \pm 0.22$
& $66.4 \pm 0.4$ \\
\textbf{Ours}
& $\mathbf{78.09 \pm 0.16}$
& $93.78 \pm 0.33$
& $\mathbf{63.0 \pm 0.6}$
& $\mathbf{81.15 \pm 0.29}$
& $\mathbf{97.12 \pm 0.23}$
& $66.6 \pm 0.6$ \\
\bottomrule
\end{tabular}
\end{table}

\textbf{Random baseline.} For each run, Random replaces the optimized synthetic inputs with a tensor whose entries are sampled independently from $\mathcal N(\mu_E,\sigma_E^2)$, where $\mu_E=\operatorname{mean}(Z_{\mathrm{init}})$ and $\sigma_E=\operatorname{std}(Z_{\mathrm{init}})$ are computed over all entries of the corresponding DASA initialization tensor. The random tensor receives no synthetic-input optimization and is evaluated with the same data budget and downstream LoRA configuration.

DASA achieves a higher mean than both GRADMM and Random in all six settings, and the highest mean among all methods in four of them. On Qwen3-8B, DASA improves MMLU from $77.41\pm0.15$ with GRADMM to $78.09\pm0.16$, and MBPP from $62.2\pm0.6$ to $63.0\pm0.6$. On Qwen3-14B, its GSM8K result is $97.12\pm0.23$, compared with $96.29\pm0.22$ for GRADMM. Real-text retains the highest mean on Qwen3-8B GSM8K and Qwen3-14B MBPP.

The standard deviations of DASA are generally comparable to those of Real-text and GRADMM, although DASA has a higher standard deviation on Qwen3-14B MBPP. These results describe the consistency of the observed mean gains across the representative configurations. Three runs per setting are insufficient to establish statistical significance.

\subsection{Task-Specialized Calibration Data}

The main experiments deliberately use generic C4 data to separate update realization from direct access to task-specific examples. We next consider a complementary setting in which informative in-domain calibration data are available. For each benchmark, calibration examples are drawn from its corresponding training split: MMLU auxiliary train, GSM8K train, MATH train, MBPP train, ARC-Challenge train, and CommonsenseQA train. The LoRA configuration, optimization budget, and evaluation protocol remain unchanged, so the only difference from the main setting is the source of the calibration data.

Table~\ref{tab:specialized_data} shows that DASA continues to outperform GRADMM in all 12 model--task comparisons across Qwen3-8B and Qwen3-14B. It achieves the strongest overall results in eight of the twelve settings, including MMLU, MATH, ARC-C, and CSQA, across both model scales. On Qwen3-14B, for example, DASA reaches 82.91 on MMLU, 89.18 on MATH, 71.59 on ARC-C, and 84.52 on CSQA, compared with 82.17, 88.38, 70.73, and 83.78 for GRADMM.

The improvement remains visible even though task-specialized calibration substantially strengthens the competing methods. This is an important complement to the C4 experiment: DASA is not effective only because generic calibration makes parameter-gradient matching difficult. Instead, its advantage persists when the calibration distribution is closely related to the target task and provides a substantially stronger adaptation signal.

Real-text remains slightly stronger on GSM8K and MBPP, where directly fine-tuning on task-specific natural-language examples is particularly competitive. Nevertheless, DASA consistently improves upon parameter-gradient matching across every evaluated configuration. Taken together with the generic-C4 results, these experiments indicate that activation-aligned synthesis is effective under substantially different calibration distributions rather than relying on a particular source of calibration text.

\begin{table}[H]
\caption{\textbf{Results with task-specialized calibration data.}
Calibration data are drawn from the corresponding training split for each benchmark: MMLU auxiliary train, GSM8K train, MATH train, MBPP train, ARC-Challenge train, and CommonsenseQA train. All adaptation methods use identical LoRA and optimization settings.}
\label{tab:specialized_data}
\centering
\scriptsize
\renewcommand{\arraystretch}{0.98}
\setlength{\tabcolsep}{1.8pt}
\begin{tabular}{l *{12}{c}}
\toprule
\multicolumn{13}{c}{\cellcolor{gray!10}\textbf{Specialized Data}} \\
\midrule
\multirow{2}{*}{\textbf{Method}} & \multicolumn{6}{c}{\cellcolor{gray!10}\textbf{Qwen3-8B}} & \multicolumn{6}{c}{\cellcolor{gray!10}\textbf{Qwen3-14B}} \\
\cmidrule(lr){2-7} \cmidrule(lr){8-13}
& MMLU & GSM8K & MATH & MBPP & ARC-C & CSQA & MMLU & GSM8K & MATH & MBPP & ARC-C & CSQA \\
\midrule
Base & 74.90 & 91.89 & 83.98 & 59.6 & 64.33 & 81.98 & 78.85 & 93.86 & 84.76 & 64.6 & 66.98 & 81.00 \\
Real-text & 78.02 & \cellcolor{blue!15}\textbf{94.62} & 87.18 & \cellcolor{blue!15}\textbf{64.0} & 67.49 & 85.01 & 81.82 & \cellcolor{blue!15}\textbf{96.36} & 88.00 & \cellcolor{blue!15}\textbf{69.0} & 70.39 & 84.03 \\
GRADMM & 78.41 & 94.39 & 87.54 & 62.6 & 67.83 & 84.77 & 82.17 & 96.13 & 88.38 & 68.6 & 70.73 & 83.78 \\
\textbf{Ours} & \cellcolor{blue!15}\textbf{79.26} & 94.54 & \cellcolor{blue!15}\textbf{88.36} & 62.8 & \cellcolor{blue!15}\textbf{68.77} & \cellcolor{blue!15}\textbf{85.42} & \cellcolor{blue!15}\textbf{82.91} & 96.29 & \cellcolor{blue!15}\textbf{89.18} & 68.8 & \cellcolor{blue!15}\textbf{71.59} & \cellcolor{blue!15}\textbf{84.52} \\
\bottomrule
\end{tabular}
\end{table}

\subsection{Computational Efficiency}
Matching parameter gradients across all Transformer layers requires repeated computation and comparison of high-dimensional gradients during synthetic-data construction. We therefore compare DASA with our all-layer implementation of GRADMM in terms of generation time and peak GPU memory under the same three calibration budgets.

All efficiency measurements are collected on a single NVIDIA A800 GPU using FP16 precision. Both methods use all Transformer layers and optimize continuous input embeddings, which are used directly for downstream LoRA adaptation without vocabulary projection. For each calibration budget, they use the same number of devices, synthetic sequences, maximum sequence length, and synthesis steps. Peak memory is recorded over the complete generation procedure. These matched settings allow us to compare the computational costs of the two synthesis procedures.

As shown in Table~\ref{tab:efficiency}, DASA substantially reduces generation time. At budgets of $50\times2048$, $100\times2048$, and $200\times2048$, GRADMM requires 28.3, 57.4, and 84.2 minutes, whereas DASA requires only 5.8, 11.6, and 23.3 minutes, respectively. This corresponds to approximately $3.6$--$4.9\times$ faster synthetic-data generation. The reduction is consistent with DASA evaluating directional effects through activation gradients, while GRADMM matches parameter gradients across all Transformer layers.

Peak GPU memory is similar for the two methods, remaining between 17.3 and 17.9 GB across all settings. Thus, DASA's measured computational benefit is a substantially shorter generation time, with comparable peak memory.
\begin{table}[H]
\caption{\textbf{Computational efficiency.}
Generation time and peak GPU memory under three calibration budgets, measured on a single NVIDIA A800 GPU using FP16. GRADMM and DASA use matched sequence counts, sequence lengths, and synthetic-data optimization steps.}
\label{tab:efficiency}
\centering
\small
\renewcommand{\arraystretch}{1.0}
\setlength{\tabcolsep}{3.2pt}
\begin{tabular}{lccc|ccc}
\toprule
& \multicolumn{3}{c|}{\cellcolor{gray!10}\textbf{Generation Time (min)} $\downarrow$}
& \multicolumn{3}{c}{\cellcolor{gray!10}\textbf{Peak Memory (GB)} $\downarrow$} \\
\cmidrule(lr){2-4}
\cmidrule(lr){5-7}
\textbf{Method}
& $\mathbf{50\times2048}$
& $\mathbf{100\times2048}$
& $\mathbf{200\times2048}$
& $\mathbf{50\times2048}$
& $\mathbf{100\times2048}$
& $\mathbf{200\times2048}$ \\
\midrule
GRADMM
& 28.3
& 57.4
& 84.2
& 17.3
& 17.7
& 17.9 \\
\textbf{DASA}
& \cellcolor{blue!15}\textbf{5.8}
& \cellcolor{blue!15}\textbf{11.6}
& \cellcolor{blue!15}\textbf{23.3}
& \cellcolor{blue!15}\textbf{17.3}
& \cellcolor{blue!15}\textbf{17.4}
& \cellcolor{blue!15}\textbf{17.9} \\
\bottomrule
\end{tabular}
\end{table}

\subsection{Sensitivity Analysis}

We examine whether DASA depends on a narrowly selected downstream
adaptation configuration. We vary the LoRA rank, which controls the
capacity of the parameter update, and the number of downstream LoRA
optimization steps. The synthetic-data optimization schedule is held
fixed. Figure~\ref{fig:sensitivity} reports validation accuracy for
adaptation on synthetic and real-text inputs under these configurations.

\begin{figure}[H]
    \centering
    \includegraphics[width=0.8\linewidth]{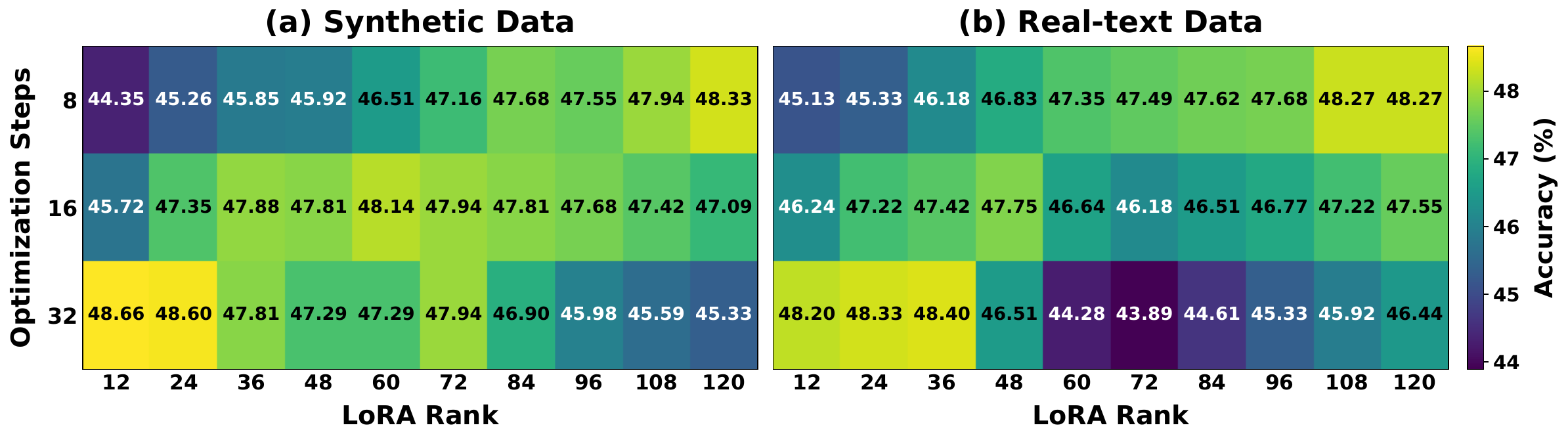}
    \caption{\textbf{Sensitivity to LoRA rank and downstream optimization
    steps.} Validation accuracy (\%) under different LoRA ranks and
    numbers of downstream LoRA optimization steps for synthetic and
    real-text adaptation.}
    \label{fig:sensitivity}
\end{figure}

DASA performs competitively across a range of configurations rather
than relying on a single choice of rank and training duration. At
smaller LoRA ranks, additional downstream steps generally improve
validation accuracy. Larger ranks reach strong performance with fewer
steps, while continued adaptation can reduce accuracy. The real-text
results show a similar dependence on rank and training duration,
although the best configurations differ.

These results suggest that LoRA rank and downstream optimization steps
are partly substitutable: a higher rank expands the space of possible
updates, while more steps allow the model to use a fixed update space
more fully. Several intermediate configurations achieve similar
accuracy, indicating that DASA's performance does not depend on a
narrowly tuned LoRA configuration.

\subsection{Qualitative Analysis of Synthetic Data}

Figure~\ref{fig:example} compares a human-readable solution with a decoded visualization of an optimized DASA sample. The visualization projects the continuous synthetic embeddings to vocabulary tokens for inspection only. Downstream adaptation uses the original continuous embeddings.

The projected sample contains fragmented words, mixed scripts, and locally incoherent token combinations. This example illustrates that the synthetic inputs need not retain a fluent natural-language surface form. Because token projection loses information from the continuous embeddings, we use downstream task performance, rather than the readability of the displayed text, to assess their training utility.

The continuous synthetic inputs are optimized in the model's embedding space to induce useful update directions, and their training value is therefore determined by how they interact with the model rather than by whether they admit an interpretable token-level rendering. The visualization in Figure~\ref{fig:example} should thus be viewed as evidence of a separation between human readability and model utility, rather than as a faithful representation of the information contained in the continuous synthetic inputs.

\begin{figure}[H]
    \centering
    \includegraphics[width=0.78\linewidth]{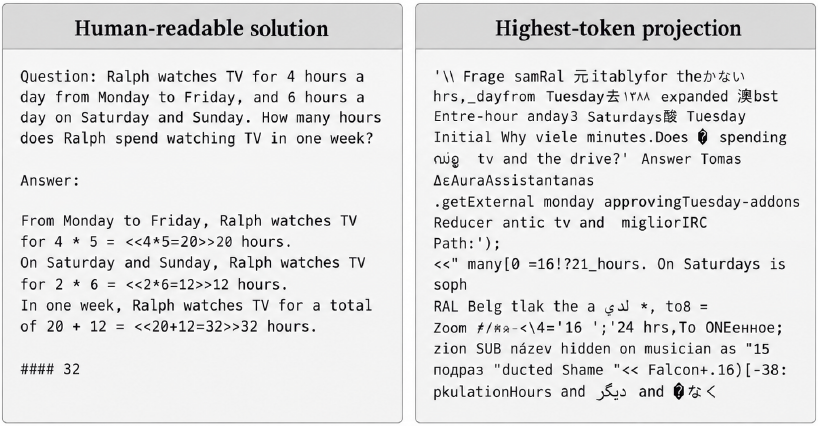}
    \caption{\textbf{Qualitative example of DASA synthetic data.} A human-readable solution is shown alongside a decoded visualization of an optimized continuous synthetic sample. }
    \label{fig:example}
\end{figure}

\section{Conclusion}

We investigated whether effective LLM fine-tuning requires human-readable training inputs. DASA constructs continuous synthetic input embeddings by optimizing the first-order effect of their induced direction on a reference loss, while keeping the pretrained model fixed during synthesis. Across six models, six benchmarks, and three data budgets, the resulting embeddings provide downstream LoRA adaptation that is often comparable to real-text fine-tuning and outperforms GRADMM in most evaluated model--task settings. Under matched synthesis workloads, DASA also reduces generation time by $3.6$--$4.9\times$ with comparable peak GPU memory. These findings show that, in the evaluated adaptation settings, useful training signals can be carried by model-conditioned continuous inputs without preserving readable language. The directional objective provides a local criterion rather than a guarantee for finite-step fine-tuning; understanding its relationship to downstream updates and transfer across models remains future work.
\section*{AI Use Statement}

We used GPT-6 solely for language polishing. All
AI-assisted edits were reviewed by the authors, who take full
responsibility for the final manuscript.

\bibliography{iclr2027_conference}

@article{DBLP:journals/corr/abs-2407-21783,
  author       = {Llama Team},
  title        = {The Llama 3 Herd of Models},
  journal      = {CoRR},
  volume       = {abs/2407.21783},
  year         = {2024},
  url          = {https://doi.org/10.48550/arXiv.2407.21783},
  doi          = {10.48550/ARXIV.2407.21783},
  eprinttype   = {arXiv},
  eprint       = {2407.21783},
  bibsource    = {dblp computer science bibliography, https://dblp.org}
}

@misc{meta2024llama32,
  author       = {{Meta}},
  title        = {Llama 3.2 Model Card},
  year         = {2024},
  howpublished = {\url{https://github.com/meta-llama/llama-models/blob/main/models/llama3_2/MODEL_CARD.md}},
  note         = {Accessed 2026-09-20}
}

@inproceedings{DBLP:conf/iclr/HendrycksBBZMSS21,
  author       = {Dan Hendrycks and
                  Collin Burns and
                  Steven Basart and
                  Andy Zou and
                  Mantas Mazeika and
                  Dawn Song and
                  Jacob Steinhardt},
  title        = {Measuring Massive Multitask Language Understanding},
  booktitle    = {9th International Conference on Learning Representations, {ICLR} 2021,
                  Virtual Event, Austria, May 3-7, 2021},
  publisher    = {OpenReview.net},
  year         = {2021},
  url          = {https://openreview.net/forum?id=d7KBjmI3GmQ},
  bibsource    = {dblp computer science bibliography, https://dblp.org}
}

@article{DBLP:journals/corr/abs-2110-14168,
  author       = {Karl Cobbe and
                  Vineet Kosaraju and
                  Mohammad Bavarian and
                  Mark Chen and
                  Heewoo Jun and
                  Lukasz Kaiser and
                  Matthias Plappert and
                  Jerry Tworek and
                  Jacob Hilton and
                  Reiichiro Nakano and
                  Christopher Hesse and
                  John Schulman},
  title        = {Training Verifiers to Solve Math Word Problems},
  journal      = {CoRR},
  volume       = {abs/2110.14168},
  year         = {2021},
  url          = {https://arxiv.org/abs/2110.14168},
  eprinttype   = {arXiv},
  eprint       = {2110.14168},
  bibsource    = {dblp computer science bibliography, https://dblp.org}
}

@article{DBLP:journals/corr/abs-2505-09388,
  author       = {Qwen Team},
  title        = {Qwen3 Technical Report},
  journal      = {CoRR},
  volume       = {abs/2505.09388},
  year         = {2025},
  url          = {https://doi.org/10.48550/arXiv.2505.09388},
  doi          = {10.48550/ARXIV.2505.09388},
  eprinttype   = {arXiv},
  eprint       = {2505.09388},
  bibsource    = {dblp computer science bibliography, https://dblp.org}
}

@inproceedings{DBLP:conf/nips/HendrycksBKABTS21,
  author       = {Dan Hendrycks and
                  Collin Burns and
                  Saurav Kadavath and
                  Akul Arora and
                  Steven Basart and
                  Eric Tang and
                  Dawn Song and
                  Jacob Steinhardt},
  editor       = {Joaquin Vanschoren and
                  Sai{-}Kit Yeung},
  title        = {Measuring Mathematical Problem Solving With the {MATH} Dataset},
  booktitle    = {Proceedings of the Neural Information Processing Systems Track on
                  Datasets and Benchmarks 1, NeurIPS Datasets and Benchmarks 2021, December
                  2021, virtual},
  year         = {2021},
  url          = {https://datasets-benchmarks-proceedings.neurips.cc/paper/2021/hash/be83ab3ecd0db773eb2dc1b0a17836a1-Abstract-round2.html},
  bibsource    = {dblp computer science bibliography, https://dblp.org}
}

@article{DBLP:journals/corr/abs-2108-07732,
  author       = {Jacob Austin and
                  Augustus Odena and
                  Maxwell I. Nye and
                  Maarten Bosma and
                  Henryk Michalewski and
                  David Dohan and
                  Ellen Jiang and
                  Carrie J. Cai and
                  Michael Terry and
                  Quoc V. Le and
                  Charles Sutton},
  title        = {Program Synthesis with Large Language Models},
  journal      = {CoRR},
  volume       = {abs/2108.07732},
  year         = {2021},
  url          = {https://arxiv.org/abs/2108.07732},
  eprinttype   = {arXiv},
  eprint       = {2108.07732},
  bibsource    = {dblp computer science bibliography, https://dblp.org}
}

@article{DBLP:journals/corr/abs-1803-05457,
  author       = {Peter Clark and
                  Isaac Cowhey and
                  Oren Etzioni and
                  Tushar Khot and
                  Ashish Sabharwal and
                  Carissa Schoenick and
                  Oyvind Tafjord},
  title        = {Think you have Solved Question Answering? Try ARC, the {AI2} Reasoning
                  Challenge},
  journal      = {CoRR},
  volume       = {abs/1803.05457},
  year         = {2018},
  url          = {http://arxiv.org/abs/1803.05457},
  eprinttype   = {arXiv},
  eprint       = {1803.05457},
  bibsource    = {dblp computer science bibliography, https://dblp.org}
}

@inproceedings{DBLP:conf/naacl/TalmorHLB19,
  author       = {Alon Talmor and
                  Jonathan Herzig and
                  Nicholas Lourie and
                  Jonathan Berant},
  editor       = {Jill Burstein and
                  Christy Doran and
                  Thamar Solorio},
  title        = {CommonsenseQA: {A} Question Answering Challenge Targeting Commonsense
                  Knowledge},
  booktitle    = {Proceedings of the 2019 Conference of the North American Chapter of
                  the Association for Computational Linguistics: Human Language Technologies,
                  {NAACL-HLT} 2019, Minneapolis, MN, USA, June 2-7, 2019, Volume 1 (Long
                  and Short Papers)},
  pages        = {4149--4158},
  publisher    = {Association for Computational Linguistics},
  year         = {2019},
  url          = {https://doi.org/10.18653/v1/n19-1421},
  doi          = {10.18653/V1/N19-1421},
  bibsource    = {dblp computer science bibliography, https://dblp.org}
}

@inproceedings{DBLP:conf/iclr/HuSWALWWC22,
  author       = {Edward J. Hu and
                  Yelong Shen and
                  Phillip Wallis and
                  Zeyuan Allen{-}Zhu and
                  Yuanzhi Li and
                  Shean Wang and
                  Lu Wang and
                  Weizhu Chen},
  title        = {LoRA: Low-Rank Adaptation of Large Language Models},
  booktitle    = {The Tenth International Conference on Learning Representations, {ICLR}
                  2022, Virtual Event, April 25-29, 2022},
  publisher    = {OpenReview.net},
  year         = {2022},
  url          = {https://openreview.net/forum?id=nZeVKeeFYf9},
  bibsource    = {dblp computer science bibliography, https://dblp.org}
}

@inproceedings{DBLP:conf/icml/NguyenLBMRM25,
  author       = {Dang Nguyen and
                  Zeman Li and
                  MohammadHossein Bateni and
                  Vahab Mirrokni and
                  Meisam Razaviyayn and
                  Baharan Mirzasoleiman},
  editor       = {Aarti Singh and
                  Maryam Fazel and
                  Daniel Hsu and
                  Simon Lacoste{-}Julien and
                  Felix Berkenkamp and
                  Tegan Maharaj and
                  Kiri Wagstaff and
                  Jerry Zhu},
  title        = {Synthetic Text Generation for Training Large Language Models via Gradient
                  Matching},
  booktitle    = {Forty-second International Conference on Machine Learning, {ICML}
                  2025, Vancouver, BC, Canada, July 13-19, 2025},
  series       = {Proceedings of Machine Learning Research},
  volume       = {267},
  publisher    = {{PMLR} / OpenReview.net},
  year         = {2025},
  url          = {https://proceedings.mlr.press/v267/nguyen25c.html},
  bibsource    = {dblp computer science bibliography, https://dblp.org}
}

@misc{wang2018dataset,
  author        = {Wang, Tongzhou and Zhu, Jun-Yan and Torralba, Antonio
                   and Efros, Alexei A.},
  title         = {Dataset Distillation},
  year          = {2018},
  howpublished  = {arXiv preprint arXiv:1811.10959},
  eprint        = {1811.10959},
  archivePrefix = {arXiv},
  url           = {https://arxiv.org/abs/1811.10959}
}

@inproceedings{zhao2021gradient,
  author    = {Zhao, Bo and Mopuri, Konda Reddy and Bilen, Hakan},
  title     = {Dataset Condensation with Gradient Matching},
  booktitle = {International Conference on Learning Representations},
  year      = {2021},
  url       = {https://openreview.net/forum?id=mSAKhLYLSsl}
}

@inproceedings{zhao2021dsa,
  author    = {Zhao, Bo and Bilen, Hakan},
  title     = {Dataset Condensation with Differentiable Siamese Augmentation},
  booktitle = {Proceedings of the 38th International Conference on Machine Learning},
  series    = {Proceedings of Machine Learning Research},
  volume    = {139},
  pages     = {12674--12685},
  year      = {2021},
  publisher = {PMLR},
  url       = {https://proceedings.mlr.press/v139/zhao21a.html}
}

@inproceedings{zhao2023distribution,
  author    = {Zhao, Bo and Bilen, Hakan},
  title     = {Dataset Condensation with Distribution Matching},
  booktitle = {Proceedings of the IEEE/CVF Winter Conference on Applications of Computer Vision},
  year      = {2023},
  url       = {https://arxiv.org/abs/2110.04181}
}

@inproceedings{cazenavette2022trajectory,
  author    = {Cazenavette, George and Wang, Tongzhou and Torralba, Antonio
               and Efros, Alexei A. and Zhu, Jun-Yan},
  title     = {Dataset Distillation by Matching Training Trajectories},
  booktitle = {Proceedings of the IEEE/CVF Conference on Computer Vision and Pattern Recognition},
  year      = {2022},
  url       = {https://arxiv.org/abs/2203.11932}
}

@inproceedings{sucholutsky2021soft,
  author    = {Sucholutsky, Ilia and Schonlau, Matthias},
  title     = {Soft-Label Dataset Distillation and Text Dataset Distillation},
  booktitle = {2021 International Joint Conference on Neural Networks (IJCNN)},
  pages     = {1--8},
  year      = {2021},
  publisher = {IEEE},
  url       = {https://ieeexplore.ieee.org/document/9533769}
}

@inproceedings{maekawa2024dilm,
  author    = {Maekawa, Aru and Kosugi, Satoshi and Funakoshi, Kotaro
               and Okumura, Manabu},
  title     = {{DiLM}: Distilling Dataset into Language Model for Text-level Dataset Distillation},
  booktitle = {Findings of the Association for Computational Linguistics: NAACL 2024},
  pages     = {3138--3153},
  year      = {2024},
  publisher = {Association for Computational Linguistics},
  doi       = {10.18653/v1/2024.findings-naacl.199},
  url       = {https://aclanthology.org/2024.findings-naacl.199/}
}

@misc{yao2025transferable,
  author        = {Yao, Rong and Hu, Hailin and Fu, Yifei and Chen, Hanting
                   and Fang, Wenyi and Du, Fanyi and Han, Kai and Wang, Yunhe},
  title         = {Transferable Text Data Distillation by Trajectory Matching},
  year          = {2025},
  howpublished  = {arXiv preprint arXiv:2504.09818},
  eprint        = {2504.09818},
  archivePrefix = {arXiv},
  url           = {https://arxiv.org/abs/2504.09818}
}

@inproceedings{duan2025unnatural,
  author    = {Duan, Keyu and Zhao, Yiran and Feng, Zhili and Ni, Jinjie
               and Pang, Tianyu and Liu, Qian and Cai, Tianle and Dou, Longxu
               and Kawaguchi, Kenji and Goyal, Anirudh and Kolter, J. Zico
               and Shieh, Michael Qizhe},
  title     = {Unnatural Languages Are Not Bugs but Features for {LLMs}},
  booktitle = {Proceedings of the 42nd International Conference on Machine Learning},
  series    = {Proceedings of Machine Learning Research},
  volume    = {267},
  pages     = {14778--14792},
  year      = {2025},
  publisher = {PMLR},
  url       = {https://proceedings.mlr.press/v267/duan25c.html}
}

@inproceedings{shin2020autoprompt,
  author    = {Shin, Taylor and Razeghi, Yasaman and Logan IV, Robert L.
               and Wallace, Eric and Singh, Sameer},
  title     = {{AutoPrompt}: Eliciting Knowledge from Language Models with Automatically Generated Prompts},
  booktitle = {Proceedings of the 2020 Conference on Empirical Methods in Natural Language Processing (EMNLP)},
  pages     = {4222--4235},
  year      = {2020},
  publisher = {Association for Computational Linguistics},
  doi       = {10.18653/v1/2020.emnlp-main.346},
  url       = {https://aclanthology.org/2020.emnlp-main.346/}
}

@inproceedings{li2021prefix,
  author    = {Li, Xiang Lisa and Liang, Percy},
  title     = {{Prefix-Tuning}: Optimizing Continuous Prompts for Generation},
  booktitle = {Proceedings of the 59th Annual Meeting of the Association for Computational Linguistics and the 11th International Joint Conference on Natural Language Processing (Volume 1: Long Papers)},
  pages     = {4582--4597},
  year      = {2021},
  publisher = {Association for Computational Linguistics},
  doi       = {10.18653/v1/2021.acl-long.353},
  url       = {https://aclanthology.org/2021.acl-long.353/}
}

@inproceedings{lester2021prompt,
  author    = {Lester, Brian and Al-Rfou, Rami and Constant, Noah},
  title     = {The Power of Scale for Parameter-Efficient Prompt Tuning},
  booktitle = {Proceedings of the 2021 Conference on Empirical Methods in Natural Language Processing},
  pages     = {3045--3059},
  year      = {2021},
  publisher = {Association for Computational Linguistics},
  doi       = {10.18653/v1/2021.emnlp-main.243},
  url       = {https://aclanthology.org/2021.emnlp-main.243/}
}

@inproceedings{liu2022ptuning,
  author    = {Liu, Xiao and Ji, Kaixuan and Fu, Yicheng and Tam, Weng
               and Du, Zhengxiao and Yang, Zhilin and Tang, Jie},
  title     = {{P-Tuning}: Prompt Tuning Can Be Comparable to Fine-tuning Across Scales and Tasks},
  booktitle = {Proceedings of the 60th Annual Meeting of the Association for Computational Linguistics (Volume 2: Short Papers)},
  pages     = {61--68},
  year      = {2022},
  publisher = {Association for Computational Linguistics},
  doi       = {10.18653/v1/2022.acl-short.8},
  url       = {https://aclanthology.org/2022.acl-short.8/}
}

@inproceedings{yin2020deepinversion,
  author    = {Yin, Hongxu and Molchanov, Pavlo and Alvarez, Jose M.
               and Li, Zhizhong and Mallya, Arun and Hoiem, Derek
               and Jha, Niraj K. and Kautz, Jan},
  title     = {Dreaming to Distill: Data-Free Knowledge Transfer via {DeepInversion}},
  booktitle = {Proceedings of the IEEE/CVF Conference on Computer Vision and Pattern Recognition},
  year      = {2020},
  url       = {https://github.com/NVlabs/DeepInversion}
}

@inproceedings{wang2022cafe,
  author    = {Wang, Kai and Zhao, Bo and Peng, Xiangyu and Zhu, Zheng
               and Yang, Shuo and Wang, Shuo and Huang, Guan and Bilen, Hakan
               and Wang, Xinchao and You, Yang},
  title     = {{CAFE}: Learning to Condense Dataset by Aligning Features},
  booktitle = {Proceedings of the IEEE/CVF Conference on Computer Vision and Pattern Recognition},
  pages     = {12186--12195},
  year      = {2022},
  doi       = {10.1109/CVPR52688.2022.01188},
  url       = {https://arxiv.org/abs/2203.01531}
}

@misc{cheng2026scaling,
  author        = {Cheng, Daning and Liu, Zeyu and Sun, Jun and Xia, Fen
                   and Zhang, Boyang and Liu, Dongping and Zhang, Yunquan},
  title         = {A Qualitative Test-Risk Mechanism for Scaling Behavior in Normalized Residual Networks},
  year          = {2026},
  howpublished  = {arXiv preprint arXiv:2605.08297},
  eprint        = {2605.08297},
  archivePrefix = {arXiv},
  url           = {https://arxiv.org/abs/2605.08297}
}

@inproceedings{houlsby2019adapters,
  author    = {Houlsby, Neil and Giurgiu, Andrei and Jastrzebski, Stanislaw
               and Morrone, Bruna and De Laroussilhe, Quentin
               and Gesmundo, Andrea and Attariyan, Mona and Gelly, Sylvain},
  title     = {Parameter-Efficient Transfer Learning for {NLP}},
  booktitle = {Proceedings of the 36th International Conference on Machine Learning},
  series    = {Proceedings of Machine Learning Research},
  volume    = {97},
  pages     = {2790--2799},
  year      = {2019},
  publisher = {PMLR},
  url       = {https://proceedings.mlr.press/v97/houlsby19a.html}
}

@inproceedings{hu2022lora,
  author    = {Hu, Edward J. and Shen, Yelong and Wallis, Phillip
               and Allen-Zhu, Zeyuan and Li, Yuanzhi and Wang, Shean
               and Wang, Lu and Chen, Weizhu},
  title     = {{LoRA}: Low-Rank Adaptation of Large Language Models},
  booktitle = {International Conference on Learning Representations},
  year      = {2022},
  url       = {https://arxiv.org/abs/2106.09685}
}
\bibliographystyle{iclr2027_conference}

\clearpage
\appendix

\section{Basic Setup and Notation}
\label{app:setup_notation}

\paragraph{Basic setup.}
Let
\(\mathcal D_{\mathrm{ref}}
=\{(x_i,y_i)\}_{i=1}^{N_{\mathrm{ref}}}\)
be the finite reference dataset. We use
\(\mathbb E_{\mathcal D_{\mathrm{ref}}}[\cdot]\)
to denote averaging under its empirical distribution, i.e.,
\(\mathbb E_{\mathcal D_{\mathrm{ref}}}[g]
=N_{\mathrm{ref}}^{-1}\sum_{i=1}^{N_{\mathrm{ref}}}g(x_i,y_i)\).
Write \(z_i=f_{\mathrm{bot}}(x_i)\) and
\(\zeta_i=\nabla_{z_i}\ell(f_{\mathrm{top}}(z_i),y_i)\).

Algorithm~1 optimizes continuous synthetic inputs \(Z^s\).
Each \(Z^s\) induces a parameter-space tangent direction
\(v(Z^s)\), whose first-order effect on the reference
activations is \(u_{v(Z^s)}(z_i)\). The synthesis objective is
the first-order variation of the empirical reference loss:
\[
\min_{Z^s}D_{\mathcal D_{\mathrm{ref}}}(v(Z^s)),
\qquad
D_{\mathcal D_{\mathrm{ref}}}(v(Z^s))
=
\frac{1}{N_{\mathrm{ref}}}
\sum_{i=1}^{N_{\mathrm{ref}}}
\zeta_i^\top u_{v(Z^s)}(z_i).
\]
A negative value indicates a locally improving direction
for a sufficiently small positive step under the stated
differentiability and realizability conditions; it does not
guarantee the outcome of finite-step fine-tuning.

For any distribution \(Q\), let
\(\mu(Q)=\mathbb E_Q[\zeta]\) and
\(\Sigma_Q=\operatorname{Cov}_Q(\zeta)\).
These statistics are used to decompose and bound \(D_Q(v)\)
below; they are not optimization targets for \(Z^s\).

\paragraph{Assumptions.}
We analyze local adaptation around the fixed pretrained model $f_{\theta_0}$. We assume differentiability of the relevant mappings and finite second moments of the activations and activation gradients, ensuring that the expectations, covariance matrices, and first-order quantities below are well defined. Additional assumptions required by individual results are stated locally.

\vspace{2pt}
\begin{tcolorbox}[
enhanced,
breakable,
colback=blue!3,
colframe=blue!45!black,
frame hidden,
borderline west={1.5pt}{0pt}{blue!50!black},
title=\textbf{Notation},
colbacktitle=blue!8,
coltitle=black,
fonttitle=\bfseries,
sharp corners,
boxsep=2pt,
left=4pt,
right=4pt,
top=3pt,
bottom=3pt
]
\footnotesize
\renewcommand{\arraystretch}{1.08}
\setlength{\tabcolsep}{2.5pt}
\begin{tabularx}{\linewidth}{
@{}
>{\raggedright\arraybackslash}p{0.20\linewidth}
>{\raggedright\arraybackslash}p{0.27\linewidth}
>{\raggedright\arraybackslash}p{0.20\linewidth}
>{\raggedright\arraybackslash}X
@{}
}
\textbf{Symbol} & \textbf{Meaning}
& \textbf{Symbol} & \textbf{Meaning} \\[1pt]
\hline
\\[-5pt]
$f_{\theta_0}
 = f_{\mathrm{top}}\circ f_{\mathrm{bot}}$
&
Fixed model decomposed at $l^*$.
&
$l^*,\,N$
&
Target layer and hidden dimension.
\\
$z_i=f_{\mathrm{bot}}(x_i)$
&
Activation of reference example $i$ at $l^*$.
&
$\zeta_i$
&
Activation gradient of reference example $i$.
\\
$\mathcal D_{\mathrm{ref}}$
&
Finite reference dataset.
&
$N_{\mathrm{ref}}$
&
Number of reference examples.
\\
$\mathbb E_{\mathcal D_{\mathrm{ref}}}$
&
Average over the reference dataset.
&
$Q$
&
A distribution, including an empirical distribution.
\\
$Z^s$
&
Continuous synthetic inputs optimized by DASA.
&
$S_{\mathrm{syn}}$
&
Stored synthetic inputs used for fine-tuning.
\\
$v(Z^s)$
&
Parameter-space tangent direction induced by $Z^s$.
&
$v$
&
A generic parameter-space tangent direction.
\\
$u_v(z)$
&
First-order activation change induced by $v$ at $z$.
&
$D_Q(v)$
&
First-order risk variation under $Q$.
\\
$\mu(Q)$
&
Mean activation gradient under $Q$.
&
$\Sigma_Q$
&
Covariance of activation gradients under $Q$.
\\
$\bar z_Q$
&
Mean activation under $Q$.
&
$\sigma_z^2(Q)$
&
Activation dispersion under $Q$.
\\
$\bar u_{v,Q}$
&
Mean first-order activation change under $Q$.
&
$c_Q(v)$
&
Covariance residual in the decomposition of $D_Q(v)$.
\\
$h_\theta$
&
Residual adaptation branch.
&
$A,B,r,\alpha$
&
LoRA factors, rank, and scaling factor.
\\
$\|\cdot\|_\sigma$
&
Spectral norm.
&
&
\\
\end{tabularx}
\end{tcolorbox}
\paragraph{Synthetic-induced tangent direction.}
Let $h_\theta$ be an auxiliary residual branch at the target layer,
parameterized around a zero-output initialization $h_0=0$.
For the current synthetic inputs $Z^s$ and their retained prediction
targets $y^s$, let $z^s$ denote the resulting target-layer activations
and let
$\zeta^s=\nabla_{z^s}\ell(f_{\mathrm{top}}(z^s),y^s)$.
Write $\mathbb E_{\mathrm{syn}}$ for the empirical average using the
same prediction-target mask and normalization as the synthetic loss.
We construct the parameter-space tangent direction and its effect on
a reference activation by
\[
v(Z^s)
=
-\mathbb E_{\mathrm{syn}}\!\left[
\left(
\left.\frac{\partial h_\theta(z^s)}{\partial\theta}
\right|_{\theta=0}
\right)^\top\zeta^s
\right],
\qquad
u_{v(Z^s)}(z)
=
\left.\frac{\partial h_\theta(z)}{\partial\theta}
\right|_{\theta=0}v(Z^s).
\]
Thus $v(Z^s)$ is constructed from the synthetic activation gradients
in the tangent space of the auxiliary residual branch. The pretrained
model and the auxiliary branch remain fixed during synthesis; only
$Z^s$ is updated. LoRA fine-tuning is performed separately after the
synthetic data have been generated.

\section{Theoretical Background}
\label{app:prior_work}

\paragraph{Auxiliary residual branch.}
\citet{cheng2026scaling} study the expansion of a fixed reference
model $f_{\mathrm{old}}^*=f_{\mathrm{top}}\circ f_{\mathrm{bot}}$.
At a selected insertion point, let $z=f_{\mathrm{bot}}(x)$ and add
an auxiliary residual branch $h_\theta$:
\begin{equation}
    f_\theta(x)
    =
    f_{\mathrm{top}}\bigl(z+h_\theta(z)\bigr).
\end{equation}
The branch has a zero-output initialization, $h_0(z)=0$, so the
expanded model initially computes the same function as
$f_{\mathrm{old}}^*$. For a parameter-space direction $v$, let
$u_v(z)$ denote the first-order representation change induced by
the branch, and define the activation gradient
\begin{equation}
    \zeta
    =
    \nabla_z\ell\bigl(f_{\mathrm{top}}(z),y\bigr).
\end{equation}
The inner product $\zeta^\top u_v(z)$ is the first-order change
in the loss of an individual example. Here $v$ is a direction in
the branch's parameter space, whereas $u_v(z)$ is its effect in
representation space. The representation change must be
realizable by the branch.

\paragraph{Population risk and the population jumpboard.}
Let $\mathcal Q$ be a data distribution. The \emph{population risk}
of a model $f$ is its expected loss under $\mathcal Q$:
\begin{equation}
    R_{\mathcal Q}(f)
    =
    \mathbb E_{(x,y)\sim\mathcal Q}
    \bigl[\ell(f(x),y)\bigr].
\end{equation}
For the expanded model along the path $\eta v$, write
$f_{\eta v}(x)
=
f_{\mathrm{top}}\bigl(z+h_{\eta v}(z)\bigr)$.
Under the differentiability and expectation-interchange
conditions of \citet{cheng2026scaling}, the initial directional
derivative of its population risk is
\begin{equation}
\begin{aligned}
    D_{\mathcal Q}(v)
    &=
    \left.
    \frac{\mathrm d}{\mathrm d\eta}
    R_{\mathcal Q}(f_{\eta v})
    \right|_{\eta=0} \\
    &=
    \mathbb E_{\mathcal Q}
    \bigl[\zeta^\top u_v(z)\bigr].
\end{aligned}
\label{eq:prior_population_direction}
\end{equation}
Condition~1 of \citet{cheng2026scaling} requires an admissible
insertion point and a realizable direction $v_{\mathrm{pop}}$
such that $D_{\mathcal Q}(v_{\mathrm{pop}})<0$. By
differentiability, there is a sufficiently small positive
$\eta$ for which
\begin{equation}
    R_{\mathcal Q}(f_{\eta v_{\mathrm{pop}}})
    <
    R_{\mathcal Q}(f_{\mathrm{old}}^*).
\end{equation}
Their Theorem~4 therefore establishes the existence of an
expanded model $\widetilde f_{\mathrm{pop}}$ with lower
population risk. We refer to this comparison model as the
\emph{population jumpboard}. It witnesses an improving point in
the expanded model class; the existence result does not imply
that a particular optimization algorithm will find it.

\paragraph{Empirical risk and the empirical jumpboard.}
For a finite sample
$S=\{(x_i,y_i)\}_{i=1}^{m}$, its \emph{empirical risk} is
\begin{equation}
    L_S(f)
    =
    \frac{1}{m}
    \sum_{i=1}^{m}\ell(f(x_i),y_i).
\end{equation}
Thus $L_{S_{\mathrm{train}}}$ and $L_{S_{\mathrm{test}}}$
denote the losses averaged over the training and test samples,
respectively; neither is identical to the population risk
$R_{\mathcal Q}$. A direction selected using
$S_{\mathrm{train}}$, followed by a small positive step,
produces an \emph{empirical jumpboard} $\widetilde f_S$.
Unlike $\widetilde f_{\mathrm{pop}}$, this model depends on the
training sample. It is also distinct from the final trained
expanded model $f_{\mathrm{new}}$.

To connect an empirical jumpboard to test loss,
\citet{cheng2026scaling} consider a setting in which the
residual branch can realize a constant representation
perturbation. Let $\mu_M$ and $g_K$ be the mean activation
gradients of independent training and test samples of sizes
$M$ and $K$. If the training-selected direction realizes
$u_v(z)=-\mu_M$, then
\begin{equation}
    \left.
    \frac{\mathrm d}{\mathrm d\eta}
    L_{S_{\mathrm{train}}}(f_{\eta v})
    \right|_{\eta=0}
    =
    -\|\mu_M\|^2,
    \qquad
    \left.
    \frac{\mathrm d}{\mathrm d\eta}
    L_{S_{\mathrm{test}}}(f_{\eta v})
    \right|_{\eta=0}
    =
    -\mu_M^\top g_K.
\end{equation}
Consequently, $\mu_M^\top g_K>0$ implies a local decrease in
test loss for a sufficiently small positive step. This
implication requires the constant perturbation to be
realizable by the branch and does not certify an arbitrary
finite step.

\paragraph{Probability of gradient alignment.}
The preceding local test-loss argument depends on the sign
of $\mu_M^\top g_K$. Theorem~5 of
\citet{cheng2026scaling} bounds the probability of
nonpositive alignment. Using the notation of Appendix~A,
let
\begin{equation}
    \mu(\mathcal Q)=\mathbb E_{\mathcal Q}[\zeta],
    \qquad
    \Sigma_{\mathcal Q}
    =\operatorname{Cov}_{\mathcal Q}(\zeta).
\end{equation}
Suppose that the training and test samples are independent
draws from $\mathcal Q$,
$\mu(\mathcal Q)\neq 0$, and
$\lambda_{\max}(\Sigma_{\mathcal Q})
\le C_\Sigma\tau^2$. Then
\begin{equation}
\begin{aligned}
    \Pr(\mu_M^\top g_K\le 0)
    \le{}&
    \frac{4C_\Sigma\tau^2}
         {M\|\mu(\mathcal Q)\|^2}
    +
    \frac{4C_\Sigma\tau^2}
         {K\|\mu(\mathcal Q)\|^2} \\
    &+
    \frac{4C_\Sigma\tau^2
          \operatorname{tr}(\Sigma_{\mathcal Q})}
         {MK\|\mu(\mathcal Q)\|^4}.
\end{aligned}
\label{eq:prior_alignment_bound}
\end{equation}
The proof separates two possible failures: the training mean
may deviate substantially from the population mean, or the
independent test mean may fluctuate against the training
direction. The covariance bound controls the probabilities of
these events. Theorem~5 concerns the sign of gradient
alignment; it does not by itself lower-bound the magnitude of
a test-loss improvement.

\paragraph{From a jumpboard to the final expanded model.}
The empirical jumpboard is a comparison model for
$f_{\mathrm{new}}$. Define its realized test-loss reduction
relative to the reference model as
\begin{equation}
    \Delta_R^{\mathrm{test}}
    =
    L_{S_{\mathrm{test}}}(f_{\mathrm{old}}^*)
    -
    L_{S_{\mathrm{test}}}(\widetilde f_S).
\end{equation}
Under the model-selection and uniform-generalization
conditions of \citet{cheng2026scaling}, their Theorem~3
combines this reduction with an additional training-loss
improvement $\Delta_{\mathrm{ERM}}$ and generalization terms
$\epsilon_M$ and $\epsilon_K$:
\begin{equation}
\begin{aligned}
    L_{S_{\mathrm{test}}}(f_{\mathrm{new}})
    \le{}&
    L_{S_{\mathrm{test}}}(f_{\mathrm{old}}^*)
    -\Delta_R^{\mathrm{test}}
    -\Delta_{\mathrm{ERM}} \\
    &+2(\epsilon_M+\epsilon_K).
\end{aligned}
\end{equation}
The comparison first relates the final model's training loss
to the jumpboard's training loss, then applies the uniform
train--test loss bound to both models. A strict improvement
of the final model follows from this bound only when the
realized gains exceed the generalization terms. The
population-risk existence result, the local empirical
jumpboard argument, and the final-model test-loss comparison
therefore establish different parts of the prior framework.

\section{Covariance Control under LoRA}
\label{app:covariance_control}

We specialize the first-order analysis to LoRA. Consider the residual
adaptation branch
\[
h_{A,B}(z)
=
\frac{\alpha}{r}BAz,
\]
where $A$ and $B$ are the LoRA factors, $r$ is the rank, and
$\alpha$ is the scaling factor. We consider the standard initialization
$B=0$.

Let $v=(V_A,V_B)$ be a tangent direction. Along the parameter path
\[
A(\eta)=A+\eta V_A,
\qquad
B(\eta)=\eta V_B,
\]
the residual branch becomes
\[
\begin{aligned}
h_{A(\eta),B(\eta)}(z)
&=
\frac{\alpha}{r}
(\eta V_B)(A+\eta V_A)z \\
&=
\eta\frac{\alpha}{r}V_BAz
+
\eta^2\frac{\alpha}{r}V_BV_Az.
\end{aligned}
\]
Taking the derivative with respect to $\eta$ at $\eta=0$ gives
\[
u_v(z)
=
\left.
\frac{d}{d\eta}
h_{A(\eta),B(\eta)}(z)
\right|_{\eta=0}
=
\frac{\alpha}{r}V_BAz.
\]
Thus, the perturbation $V_A$ has no first-order contribution under
the initialization $B=0$.

We next decompose the first-order variation. By definition,
\[
D_Q(v)
=
\mathbb E_Q[\zeta^\top u_v].
\]
Write
\[
\zeta
=
\mu(Q)+(\zeta-\mu(Q))
\]
and
\[
u_v
=
\bar u_{v,Q}
+
(u_v-\bar u_{v,Q}).
\]
Substituting both decompositions gives
\[
\begin{aligned}
D_Q(v)
&=
\mathbb E_Q
\Big[
\big(\mu(Q)+\zeta-\mu(Q)\big)^\top
\big(\bar u_{v,Q}+u_v-\bar u_{v,Q}\big)
\Big] \\
&=
\mu(Q)^\top\bar u_{v,Q}
+
\mu(Q)^\top
\mathbb E_Q[u_v-\bar u_{v,Q}] \\
&\quad
+
\mathbb E_Q[\zeta-\mu(Q)]^\top
\bar u_{v,Q} \\
&\quad
+
\mathbb E_Q
\left[
(\zeta-\mu(Q))^\top
(u_v-\bar u_{v,Q})
\right].
\end{aligned}
\]
Since
\[
\mathbb E_Q[\zeta-\mu(Q)]=0
\]
and
\[
\mathbb E_Q[u_v-\bar u_{v,Q}]=0,
\]
the two cross terms vanish. Therefore,
\[
D_Q(v)
=
\mu(Q)^\top\bar u_{v,Q}
+
c_Q(v),
\]
where
\[
c_Q(v)
=
\mathbb E_Q
\left[
(\zeta-\mu(Q))^\top
(u_v-\bar u_{v,Q})
\right].
\]

\begin{theorem}[Covariance Control]
\label{thm:covariance_control}
Under the assumptions in Appendix~\ref{app:setup_notation},
\[
\boxed{
|c_Q(v)|
\le
\sqrt{\operatorname{tr}(\Sigma_Q)}
\frac{\alpha}{r}
\|V_BA\|_\sigma
\sigma_z(Q).
}
\]
\end{theorem}

\begin{proof}
By definition,
\[
c_Q(v)
=
\mathbb E_Q
\left[
(\zeta-\mu(Q))^\top
(u_v-\bar u_{v,Q})
\right].
\]
Applying Cauchy--Schwarz twice gives
\[
\begin{aligned}
|c_Q(v)|
&\le
\mathbb E_Q
\left[
\|\zeta-\mu(Q)\|_2
\|u_v-\bar u_{v,Q}\|_2
\right] \\
&\le
\sqrt{\mathbb E_Q\|\zeta-\mu(Q)\|_2^2}
\sqrt{\mathbb E_Q\|u_v-\bar u_{v,Q}\|_2^2}.
\end{aligned}
\]

For the first factor, since
\[
\Sigma_Q
=
\mathbb E_Q
\left[
(\zeta-\mu(Q))(\zeta-\mu(Q))^\top
\right],
\]
we have
\[
\mathbb E_Q\|\zeta-\mu(Q)\|_2^2
=
\operatorname{tr}(\Sigma_Q).
\]

For the second factor, using
\[
u_v(z)=\frac{\alpha}{r}V_BAz
\]
and
\[
\bar u_{v,Q}
=
\frac{\alpha}{r}V_BA\bar z_Q,
\]
we obtain
\[
u_v-\bar u_{v,Q}
=
\frac{\alpha}{r}V_BA(z-\bar z_Q).
\]
Hence,
\[
\begin{aligned}
\mathbb E_Q
\|u_v-\bar u_{v,Q}\|_2^2
&=
\frac{\alpha^2}{r^2}
\mathbb E_Q
\|V_BA(z-\bar z_Q)\|_2^2 \\
&\le
\frac{\alpha^2}{r^2}
\|V_BA\|_\sigma^2
\mathbb E_Q
\|z-\bar z_Q\|_2^2 \\
&=
\frac{\alpha^2}{r^2}
\|V_BA\|_\sigma^2
\sigma_z^2(Q).
\end{aligned}
\]

Substituting the two bounds above yields
\[
|c_Q(v)|
\le
\sqrt{\operatorname{tr}(\Sigma_Q)}
\frac{\alpha}{r}
\|V_BA\|_\sigma
\sigma_z(Q).
\]
\end{proof}

\paragraph{Implication.}
From the decomposition above,
\[
D_Q(v)
=
\mu(Q)^\top\bar u_{v,Q}+c_Q(v).
\]
Since $c_Q(v)\le |c_Q(v)|$, Theorem~\ref{thm:covariance_control}
implies
\[
D_Q(v)
\le
\mu(Q)^\top\bar u_{v,Q}
+
\sqrt{\operatorname{tr}(\Sigma_Q)}
\frac{\alpha}{r}
\|V_BA\|_\sigma
\sigma_z(Q).
\]
Therefore, if
\[
-\mu(Q)^\top\bar u_{v,Q}
>
\sqrt{\operatorname{tr}(\Sigma_Q)}
\frac{\alpha}{r}
\|V_BA\|_\sigma
\sigma_z(Q),
\]
then the right-hand side is strictly negative, and hence
\[
D_Q(v)<0.
\]

For a fixed reference distribution $Q$, the activation statistics
$\mu(Q)$, $\Sigma_Q$, and $\sigma_z(Q)$ are fixed. In contrast, the
direction $v=v(Z^s)$ is induced by the synthetic inputs and can therefore
be shaped during synthesis. Optimizing $Z^s$ can increase the favorable
mean component $-\mu(Q)^\top\bar u_{v,Q}$ while controlling the geometry
of the induced LoRA direction through $\|V_BA\|_\sigma$. The bound
therefore provides a mechanism by which synthetic-data optimization can
favor directions whose mean improvement dominates the covariance
residual, whereas a fixed real-text corpus does not directly optimize
this trade-off.

\section{Training Details}
\label{app:training_details}

\paragraph{Downstream LoRA adaptation.}
We freeze all pretrained backbone parameters and perform downstream
adaptation using LoRA with rank $r=32$, scaling factor $\alpha=32$,
and dropout $0.05$. LoRA modules are applied to the query, key,
value, output, gate, up, and down projection layers. We optimize
the LoRA parameters for 32 optimizer steps using AdamW with a
learning rate of $2\times10^{-4}$,
$(\beta_1,\beta_2)=(0.9,0.999)$, $\epsilon=10^{-8}$,
weight decay $0.01$, and gradient clipping at $1.0$. We use a
cosine learning-rate schedule with a $3\%$ warmup ratio, BF16
precision, a maximum sequence length of 2048, a per-device batch
size of 1, and gradient accumulation over 8 steps. Unless
otherwise specified, all experiments use the same downstream
adaptation configuration across models, tasks, and methods.

\paragraph{Synthetic-data optimization.}
DASA initializes continuous synthetic inputs from the input
embeddings of sampled reference sequences and retains their
next-token prediction targets. The pretrained model and prediction
targets remain fixed during synthesis; only the synthetic inputs
are optimized. We update these inputs for $T_{\mathrm{syn}}=60$
steps using SGD with learning rate $\eta_{\mathrm{syn}}=0.1$,
zero momentum, zero weight decay, and gradient clipping at $1.0$.
The resulting synthetic tensors are stored in BF16 and used
directly as input embeddings during downstream LoRA adaptation.
DASA and GRADMM use identical sequence counts, sequence lengths,
synthesis steps, and downstream adaptation configurations.

\end{document}